\documentclass[letterpaper]{article}
\usepackage{aaai2027}
\usepackage[hyphens]{url}
\usepackage{natbib}
\usepackage{amsmath,amssymb,amsthm,mathtools}
\usepackage{booktabs}
\newtheorem{theorem}{Theorem}
\newtheorem{lemma}[theorem]{Lemma}
\newtheorem{corollary}[theorem]{Corollary}
\newtheorem{proposition}[theorem]{Proposition}
\newtheorem{definition}{Definition}
 
\newcommand{\E}{\mathbb{E}}
\newcommand{\Cadapt}{\mathcal{C}_{\mathrm{adapt}}}
\newcommand{\Roracle}{R^{\star}}
\newcommand{\Pp}{\mathbb{P}}
\newcommand{\Aad}{\mathcal{A}_{\mathrm{ad}}}
\newcommand{\Aoracle}{\mathcal{A}_{\mathrm{or}}}
\newcommand{\Holder}{\mathcal{H}}
\newcommand{\TV}{\operatorname{TV}}
\newcommand{\KL}{\operatorname{KL}}
 
\title{The Cost of Adaptivity:\\
Matching Lower Bounds Across Learning Problems}
\newcommand{\supp}[1]{Appendix~#1 of the supplement}

\author{
    Ibne Farabi Shihab,\textsuperscript{\rm 1} Adria Binte Habib\textsuperscript{\rm 2}
}
\affiliations{
    \textsuperscript{\rm 1}Department of Computer Science, Iowa State University, USA\\
    \textsuperscript{\rm 2}Department of Computer Science \& Engineering, Independent University of Bangladesh, Bangladesh\\
    ishihab@iastate.edu , binte.adria708@gmail.com
}

\begin{document}
\maketitle

\begin{abstract}
Adaptive procedures must work without nuisance information an oracle may use, such as a gradient scale or
a smoothness index, and robust procedures may also have to answer queries whose coordinate and inspection
time are chosen only after the data are seen. Such comparisons are meaningful only when the oracle
advantage and the validity contract are stated explicitly. We formalize nuisance adaptation through a
slice-normalized minimax ratio that retains the worst-case instance inside each nuisance slice, and
separately define the robustness cost of expanding from one preannounced Gaussian query to arbitrary
post-hoc inspection. Our main result is a finite-horizon composition law for Gaussian certification: from
$M$ independent coordinates, a familywise certifier protecting every coordinate and every time up to $T$
pays optimal normalized squared half-width $\Theta_\alpha(\log(eM)+\log\log(e^eT))$ within the
sample-mean-centered rectangular class. Epoch stitching gives the upper bound; independent Gaussian block
increments across coordinates and geometric time scales give a matching lower bound that already holds on
a geometric checkpoint grid and forces quantiles of the realized maximum width, so the selection and
stopping taxes add. Two benchmark regimes complete the picture: unknown gradient scale in online convex
optimization has constant cost, whereas pointwise adaptation over a continuum of nested H\"older classes
costs $\Theta((\log n/\log\log n)^{s_1/(2s_1+1)})$. Cast as model monitoring, the composition law lets an
analyst inspect any of $M$ slice metrics at any data-dependent time: the naive fixed-query band's selected
coverage degrades sharply---to $0.30$ at $M=1$ and to zero for $M\ge10$---while the epoch-stitched
certifier holds familywise coverage at an additive iterated-logarithm width cost. Experiments put both sharp predictions at risk of refutation; both survive.
\end{abstract}
 
\section{Introduction}
Learning algorithms are routinely tuned to quantities that are unavailable at deployment time. An online
optimizer may be analyzed with a gradient bound and a nonparametric estimator with a smoothness index. An
oracle may use such information; an adaptive procedure may not. A related but distinct burden appears when
a statistical certificate must remain valid for a query chosen only after the data are seen. These two
situations should not be conflated: a nuisance-adaptation ratio is interpretable only when oracle and
adaptive procedures solve the same task under the same loss and validity contract, whereas a robustness
comparison must state explicitly that one side protects a larger query family.
 
This paper develops a common normalization for the first problem and an explicit companion functional for
the second. A nuisance parameter indexes a family of instance classes; the oracle minimax risk is computed
within each slice, and one nuisance-agnostic procedure is compared with that slice-specific benchmark
before a supremum is taken over slices. Keeping the within-slice worst case is essential, since a
smoothness index does not identify a regression function and a gradient bound does not identify an online
loss sequence.
 
Two sufficient principles organize the framework. Under \emph{symmetry}, if changing the unknown quantity
induces a bijection of instance classes and a conjugation of oracle algorithms under which risk scales
homogeneously, one equivariant nuisance-agnostic algorithm is simultaneously competitive on every slice.
Under \emph{statistical ambiguity}, if two slices are hard to distinguish but require incompatible
low-risk decisions, a Le~Cam reduction gives a normalized lower bound; an asymmetric constrained-risk
lemma complements this by transferring very small risk on one slice into a lower bound on a nearby
alternative.
 
The main new theorem concerns a different but precisely defined robustness expansion. We observe $M$
independent Gaussian coordinates over time. A fixed-query benchmark is told in advance which coordinate
will be reported and when it will be inspected. A familywise certifier must instead maintain
sample-mean-centered intervals covering all $M$ means at all times through $T$. We first prove that
familywise coverage is equivalent, pointwise in the unknown mean, to coverage against every measurable
post-hoc coordinate--time selector. We then define the normalized cost of expanding from one
preannounced query to that robust contract and prove it is
$\Theta_\alpha(\log(eM)+\log\log(e^eT))$. The upper bound allocates error probability across coordinates
and geometric time epochs. The lower bound is not a union-bound converse: under the null mean, coverage
forces all normalized partial sums inside the radius envelope, and differences across quartic time points
produce $M\lfloor\log_4T\rfloor$ independent standard Gaussian block increments whose maximum forces the
same additive $\log M+\log\log T$ term. The argument works at every finite horizon $T\ge16$, uses only a
geometric checkpoint grid, and also lower-bounds fixed quantiles of the realized maximum width, so the
result is not an artifact of a deterministic pathwise supremum.
 
Two benchmark regimes test the normalization. For online convex optimization, an unknown global
gradient scale is a true symmetry. Algorithm conjugation gives exact oracle-risk homogeneity, and
scale-free online learning yields $\Cadapt=\Theta(1)$. Pointwise smoothness adaptation is subtler. A
coarse Lepski bound pays $(\log n)^{s/(2s+1)}$ on every displayed slice, but the sharp adaptive profile
of \citet{lepski1997pointwise} degenerates as $s$ approaches the smooth endpoint. Optimizing that profile
under the paper's infimum--supremum definition, and matching it by a moving local-alternative argument,
gives $\Cadapt{}_{n}^{\rm pt}([s_0,s_1])=\Theta((\log n/\log\log n)^{s_1/(2s_1+1)})$. The hardest witness
is not the endpoint itself but a class whose exponent lies $\Theta(1/\log\log n)$ below it, which is why
maximizing one estimator's coarse upper profile over the smoothness scale gives the wrong global cost.
 
The novelty claim is deliberately narrow, and the Related Work section states it precisely: we do not
claim invention of the adaptive-to-oracle ratio, of simultaneous confidence sequences, or of
post-selection inference. The new mathematical claim is the sharp finite-horizon selection--stopping
robustness law for the stated centered rectangular family, together with the independent-block lower
bound proving additive composition.
 
Experiments demonstrate and stress-test these predictions. Cast as model monitoring, the composition law
shows that a naive fixed-query band's selected coverage degrades to $0.30$ at $M=1$ and to zero for
$M\ge10$ under worst-slice inspection, while the epoch-stitched certifier holds familywise coverage at an additive-iterated-logarithm width cost. On a
grid to $M=1024$, $T\approx10^6$ the minimal familywise squared radius is affine in
$(\log(eM),\log\log(e^eT))$ with $R^2=0.998$ and negligible interaction, its $\log(eMK)$ slope lies in the
theorem's constant band, and coverage is uniform over adversarial means. The pointwise worst slice
migrates interiorly toward $s_1$ at the predicted $1/\log\log n$ rate, and the scale-free OCO cost is a
flat $\Theta(1)$ regret ratio under an adversary realizing the lower bound.
 
\section{Setup and Related Work}\label{sec:def}
Let $m$ denote the statistical budget, such as sample size or horizon. A nuisance parameter
$\lambda\in\Lambda$ indexes a slice $\mathfrak{P}_{\lambda,m}$. The index may describe a property of the
data-generating instance, such as smoothness or scale. For an algorithm $A$ and instance $P$, let
$L_m(A,P)\ge0$ be the loss. The oracle knows $\lambda$ but not the full instance
$P\in\mathfrak{P}_{\lambda,m}$, whereas the adaptive learner uses one algorithm across all slices.
 
\begin{definition}[Slice oracle risk and cost of adaptivity]\label{def:cost}
For each $\lambda\in\Lambda$, define
\[
\Roracle_m(\lambda)
=\inf_{A\in\Aoracle(\lambda)}\ \sup_{P\in\mathfrak{P}_{\lambda,m}}
\E_P[L_m(A,P)].
\]
Assume $0<\Roracle_m(\lambda)<\infty$ on the slices under comparison and
$\Aad\subseteq\Aoracle(\lambda)$ for every $\lambda$. The slice-normalized cost is
\[
\Cadapt{}_{m}(\Lambda)
=\inf_{A\in\Aad}\ \sup_{\lambda\in\Lambda}
\frac{\displaystyle\sup_{P\in\mathfrak{P}_{\lambda,m}}\E_P[L_m(A,P)]}
{\Roracle_m(\lambda)}.
\]
The associated slice profile of a fixed adaptive algorithm is the ratio before taking the supremum over
$\lambda$.
\end{definition}
 
The definition deliberately excludes an arbitrary additive stabilizer. If an oracle risk is zero, a
multiplicative ratio is not the right object; that slice requires an additive excess-risk or regret
comparison. The Gaussian selection--stopping theorem is deliberately kept outside
Definition~\ref{def:cost}: its robust side protects all coordinate--time inspections simultaneously,
whereas its denominator is the optimal preannounced-query benchmark. Conflating those validity contracts
would make the fixed-width procedure $q_\alpha/\sqrt t$ a spurious ratio-one counterexample. We therefore
introduce a separate post-hoc inspection robustness cost below.
 
\begin{proposition}[Orientation]\label{prop:orientation}
Under Definition~\ref{def:cost}, $\Cadapt{}_{m}(\Lambda)\ge1$.
\end{proposition}
 
\begin{proof}
Since $\Aad\subseteq\Aoracle(\lambda)$, $\Roracle_m(\lambda)$ is no larger than the worst-case risk of any
$A\in\Aad$ on that slice, so each ratio is at least one; suprema and infima preserve this.
\end{proof}
 
\subsection{The comparison contract}
The oracle and adaptive sides of Definition~\ref{def:cost} must share the same instance class, loss,
observation model, horizon, and validity requirement; their only difference is access to $\lambda$. This
rule prevents a common category error: a fixed-time Gaussian confidence interval and an anytime-valid
confidence sequence do not differ only in knowledge of a noise scale but satisfy different coverage
requirements, and comparing a worst-case OCO guarantee with a realized-energy coefficient does not
produce a formal adaptivity ratio. A robustness expansion can still be scientifically useful, but must be
named as such rather than presented as a same-contract nuisance comparison, as the Gaussian composition
theorem does explicitly. The commonality across the paper is the normalization discipline, not a claim
that regret, estimation risk, and confidence width are the same mathematical object.
 
\subsection{Related Work}
The decision-theoretic basis of the comparison is classical: oracle and adaptive procedures are minimax
decision rules on different information classes, and the lower-bound toolkit includes two-point, Le~Cam,
Assouad, and Fano arguments
\citep{wald1950statistical,lecam1986asymptotic,lecam1973convergence,assouad1983deux,yu1997assouad,tsybakov2009nonparametric,wainwright2019high}.
The framework contribution is not a new testing inequality but the requirement that the adaptive and
oracle sides be normalized slice by slice under the same task, loss, observation model, and validity
contract.
 
Bandit theory provides many of the closest examples of parameter dependence, from gap-dependent and
minimax regret analyses \citep{lai1985asymptotically,auer2002finite,lattimore2020bandit} to best-arm
identification and sequential exploration, which expose related stopping and confidence effects
\citep{jamieson2014lil,kaufmann2016complexity,garivier2016optimal,audibert2010best}; iterated-logarithm
and time-uniform boundaries arise in \citet{darling1967iterated} and \citet{howard2021time}. These works
motivate the temporal axis of our certification theorem but do not turn it into an unknown-variance
bandit regret theorem, and noise-adaptive confidence sets require their own reduction and objective
\citep{jun2024noiseadaptive}.
 
Adaptation to unknown smoothness is classical in nonparametric statistics
\citep{lepski1991problem,lepski1997optimal,donoho1995adapting,cai2005adaptive,gine2016mathematical,ibragimov1981statistical}.
For pointwise estimation, the logarithmic penalty and the sharper smoothness-dependent adaptive profile
are treated directly in \citet{lepski1997pointwise}; nonasymptotic oracle comparisons are developed in
\citet{spokoiny2009parameter}. Smoothness adaptation also appears in continuous-armed bandits
\citep{locatelli2018adaptivity}. Our pointwise theorem is not a new fixed-slice minimax rate. It optimizes
the sharp classical profile under Definition~\ref{def:cost} and supplies a direct constrained-risk lower
bound for the resulting global ratio.
 
Online learning provides the free-adaptation benchmark. AdaGrad, adaptive bound optimization,
scale-free online linear optimization, coin betting, and the broader OCO literature show that important
loss scales can be removed without asymptotic cost
\citep{duchi2011adaptive,mcmahan2010adaptive,orabona2015scale,orabona2016coin,cesabianchi2006prediction,hazan2016introduction,shalev2012online,orabona2019modern}.
The OCO result below uses this established symmetry rather than claiming a new optimizer.
 
Other literatures use ``adaptivity'' for a different axis: choosing later queries from earlier answers.
That notion is studied in property testing, submodular optimization, and adaptive sensing
\citep{canonne2018adaptivity,golovin2011adaptive},
with lower-bound perspectives from Yao's principle, communication complexity, and
statistical--computational gaps
\citep{yao1977probabilistic},
and related oracle gaps in active learning, distributed estimation, and reinforcement learning
\citep{hanneke2014theory,castro2008minimax,zhang2013information,azar2017minimax,jaksch2010near}.
Our hidden nuisance parameter is a third axis, and the comparison contract is what lets these
distinctions remain explicit.
 
The novelty claim is deliberately narrow. A recent transfer-learning paper defines an intrinsic cost of
adaptation through a worst-case adaptive-risk to oracle-risk ratio, so
we do not claim invention of the ratio. Time-uniform vector confidence regions, post-selection inference,
and analyses separating choosing and stopping also precede this work
\citep{chugg2025timeuniform}. The new mathematical claim is the sharp
finite-horizon oracle-normalized selection--stopping composition law for the centered rectangular class,
together with the block-increment lower bound proving that the coordinate and inspection-time taxes add.
 
\subsection{Two General Principles}\label{sec:principles}
The examples are organized by two sufficient conditions. The first gives constant cost under a nuisance
symmetry. The second turns decision incompatibility and statistical closeness into normalized lower
bounds.
 
\begin{theorem}[Risk-homogeneous nuisance symmetry]\label{thm:symmetry}
Fix a reference slice $\lambda_0$. Suppose that for every $\lambda\in\Lambda$ there are bijections
\[
T_\lambda:\mathfrak{P}_{\lambda_0,m}\to\mathfrak{P}_{\lambda,m},
\qquad
U_\lambda:\Aoracle(\lambda_0)\to\Aoracle(\lambda),
\]
and a scalar $a_\lambda>0$ such that, for every oracle algorithm $A$ and every
$P\in\mathfrak{P}_{\lambda_0,m}$,
\[
\E_{T_\lambda P}[L_m(U_\lambda A,T_\lambda P)]
=a_\lambda\E_P[L_m(A,P)].
\]
Then
\[
\Roracle_m(\lambda)=a_\lambda\Roracle_m(\lambda_0).
\]
Assume in addition that a single $A^{\rm sf}\in\Aad\subseteq\Aoracle(\lambda)$ for every $\lambda$ is
equivariant in risk,
\[
\E_{T_\lambda P}[L_m(A^{\rm sf},T_\lambda P)]
=a_\lambda\E_P[L_m(A^{\rm sf},P)],
\]
and satisfies
\[
\sup_{P\in\mathfrak{P}_{\lambda_0,m}}\E_P[L_m(A^{\rm sf},P)]
\le K\Roracle_m(\lambda_0).
\]
Then $\Cadapt{}_{m}(\Lambda)\le K$.
\end{theorem}
 
Proof in \supp{A}.
 
\begin{theorem}[Two-slice decision-relevant ambiguity]\label{thm:twoslice}
Consider two slices $\lambda_0,\lambda_1$ and instances
$P_j\in\mathfrak{P}_{\lambda_j,m}$. A fixed adaptive algorithm acts through one common, possibly
randomized Markov kernel from the observed data to its output under both instances. In an interactive or
sequential problem, $P_j$ denotes the law of the full observed transcript. Let $B_0$ and $B_1$ be
disjoint sets of outputs. Suppose that on instance $P_j$ the loss is at least $\Delta_j>0$ whenever the
output is outside $B_j$. Then every adaptive algorithm satisfies
\[
\max_{j\in\{0,1\}}
\frac{\E_{P_j}[L_m(A,P_j)]}{\Roracle_m(\lambda_j)}
\ge
\frac{1-\TV(P_0,P_1)}
{\Roracle_m(\lambda_0)/\Delta_0+\Roracle_m(\lambda_1)/\Delta_1}.
\]
By Pinsker's inequality, $\TV(P_0,P_1)$ may be replaced by
$\sqrt{\KL(P_0\|P_1)/2}$ whenever the latter is smaller than one.
\end{theorem}
 
Proof in \supp{A}.
 
\begin{lemma}[Asymmetric constrained-risk transfer]\label{lem:constrained-risk}
Let $P_1\ll P_0$, let $\vartheta(P)$ be a real target, and let $\widehat\vartheta$ be any estimator. Write
\[
r_j=\left\{\E_{P_j}(\widehat\vartheta-\vartheta(P_j))^2\right\}^{1/2},
\qquad
\Delta=|\vartheta(P_1)-\vartheta(P_0)|,
\]
and
\[
I(P_1,P_0)
=\left\{\E_{P_0}\left(\frac{dP_1}{dP_0}\right)^2\right\}^{1/2}.
\]
Then
\[
r_1\ge\Delta-r_0I(P_1,P_0).
\]
\end{lemma}
 
Proof in \supp{A}.
 
Theorem~\ref{thm:symmetry} formalizes the free side of the rescaling intuition.
Theorem~\ref{thm:twoslice} is the symmetric lower-bound template, while
Lemma~\ref{lem:constrained-risk} transfers unusually small risk on one slice into a lower bound on a
nearby alternative. The latter is used in the pointwise proof. None of the three is an if-and-only-if
criterion.
 
\section{Three Regimes and Matching Bounds}\label{sec:bounds}
Table~\ref{tab:summary} states the three comparisons before giving their details. The OCO and pointwise
rows instantiate Definition~\ref{def:cost}. The Gaussian row uses the separate robustness functional in
Definition~\ref{def:inspection-cost}.
 
\begin{table*}[t]
\centering
\small
\setlength{\tabcolsep}{5pt}
\begin{tabular}{p{0.16\linewidth}p{0.14\linewidth}p{0.23\linewidth}p{0.24\linewidth}p{0.10\linewidth}}
\toprule
Problem & Hidden quantity & Fixed comparison contract & Normalized result & Status \\
\midrule
Scale-free OCO & Gradient scale $G$ & Worst-case regret, same horizon and domain & $\Theta(1)$ & Prior results + framework \\
Gaussian familywise certification & Post-hoc coordinate and time inspection & Centered rectangles; fixed-query benchmark versus robustness to all coordinate--time inspections & $\Theta_\alpha(\log(eM)+\log\log(e^eT))$ & New matching theorem \\
Pointwise H\"older estimation & Smoothness $s$ & Pointwise RMSE over a continuum of nested H\"older balls & $\Theta\bigl((\tfrac{\log n}{\log\log n})^{\beta_1}\bigr)$ & Sharp derived cost \\
\bottomrule
\end{tabular}
\caption{Two nuisance-adaptation comparisons and one explicitly separated robustness expansion, with
$\beta_1=s_1/(2s_1+1)$. The Gaussian row concerns familywise coordinate--time certification, not
unknown-variance bandit regret.}
\label{tab:summary}
\end{table*}
 
\subsection{Unknown gradient scale: constant-cost OCO}
Let $\mathcal{K}\subset\mathbb{R}^d$ be convex and compact with Euclidean diameter $D>0$. At round
$t$, the learner chooses $x_t\in\mathcal{K}$ and then observes a linear loss vector $g_t$ with
$\|g_t\|_2\le G$. The regret is
\[
\operatorname{Reg}_T
=\sum_{t=1}^T\langle g_t,x_t\rangle
-\min_{u\in\mathcal{K}}\sum_{t=1}^T\langle g_t,u\rangle.
\]
The nuisance slice indexed by $G$ contains all sequences satisfying the norm bound. The oracle may tune
to $G$; the adaptive algorithm must not use it.
 
Scale-free online linear optimization provides an algorithm whose decisions are unchanged when all loss
vectors are multiplied by the same positive constant and whose regret is bounded by a universal constant
times $D\sqrt{\sum_t\|g_t\|_2^2}$ \citep{orabona2015scale,duchi2011adaptive}. Standard minimax lower
bounds are of order $GD\sqrt{T}$ on any nontrivial bounded domain. Conjugating algorithms by
$g_t\mapsto ag_t$ proves the exact identity $\Roracle_T(aG)=a\Roracle_T(G)$; \supp{B} gives
both directions. The instance classes may be nested, and deterministic sequences are represented as
point-mass distributions in Definition~\ref{def:cost}. Theorem~\ref{thm:symmetry} therefore applies.
 
\begin{corollary}[Scale-free OCO]\label{cor:oco}
For $0<G_{\min}\le G_{\max}<\infty$,
\[
\Cadapt{}_{T}^{\rm oco}([G_{\min},G_{\max}])=\Theta(1).
\]
The constants depend on the chosen norm and scale-free algorithm, not on $G$ or $T$.
\end{corollary}
 
The statement intentionally avoids an unsupported exact constant. A bound such as
$\sqrt{2}D\sqrt{\sum_t\|g_t\|^2}$ is algorithm- and normalization-specific and must be tied to an
explicit update rule, geometry, and zero-gradient convention before it is presented as canonical.
 
\subsection{Post-hoc coordinate and time: an additive robustness law}
Let
\[
X_1,X_2,\ldots\overset{\rm iid}{\sim}N(\mu,I_M),
\qquad \mu\in\mathbb{R}^M,
\]
and write $\bar X_{j,t}=t^{-1}\sum_{i=1}^tX_{ij}$. A preannounced fixed-query benchmark is told a
coordinate $j$ and inspection time $t$ and needs coverage only for that pair. Within the
sample-mean-centered class, its smallest worst-case half-width is
\[
w_t^{\rm or}=\frac{q_\alpha}{\sqrt t},
\qquad q_\alpha=\Phi^{-1}(1-\alpha/2).
\]
 
The robust certifier is told neither coordinate nor inspection time. At every $t\le T$ it outputs
rectangles with coordinate intervals
\[
C_{j,t}(X_{1:t})=
[\bar X_{j,t}-w_{j,t}(X_{1:t}),\bar X_{j,t}+w_{j,t}(X_{1:t})]
\]
and must satisfy
\[
\inf_{\mu\in\mathbb{R}^M}
\Pp_\mu\!\left(
\mu_j\in C_{j,t}\text{ for every }j\le M\text{ and }1\le t\le T
\right)\ge1-\alpha.
\]
Let $\mathfrak I$ be the class of all measurable selectors from the full trajectory $X_{1:T}$ to one
pair in $[M]\times[T]$.
 
\begin{lemma}[Familywise coverage and adversarial inspection]\label{lem:inspection-equivalence}
For every fixed $\mu$ and every collection of measurable intervals $\{C_{j,t}\}$,
\[
\Pp_\mu\!\left(\forall j,t:\mu_j\in C_{j,t}\right)
=
\inf_{\iota\in\mathfrak I}
\Pp_\mu\!\left(
\mu_{J_\iota(X)}\in C_{J_\iota(X),\tau_\iota(X)}
\right),
\]
where $\iota(X)=(J_\iota(X),\tau_\iota(X))$. Hence familywise coverage at level $1-\alpha$ is equivalent
to selected coverage at level $1-\alpha$ uniformly over all measurable post-hoc selectors.
\end{lemma}
 
Proof in \supp{A}.
 
For a robust procedure $A$, define its deterministic normalized radius envelope
\[
\rho(A;M,T)
=\sup_{x_{1:T}}\max_{j\le M,\,1\le t\le T}
\sqrt t\,w_{j,t}(x_{1:t}).
\]
 
\begin{definition}[Post-hoc inspection robustness cost]\label{def:inspection-cost}
Within the sample-mean-centered rectangular class, let
\[
\mathcal C^{\rm cert}_{M,T}
=\inf_{A\in\mathcal A_{\rm fw}(M,T,\alpha)}
\frac{\rho(A;M,T)^2}{q_\alpha^2},
\]
where $\mathcal A_{\rm fw}(M,T,\alpha)$ is the familywise-valid class above.
\end{definition}
 
This is intentionally not identified with Definition~\ref{def:cost}. The numerator expands the validity
contract from one preannounced query to robustness against every coordinate--time inspection. The simple
procedure $w_{j,t}\equiv q_\alpha/\sqrt t$ has exact marginal coverage at each deterministic pair, but it
fails Lemma~\ref{lem:inspection-equivalence}; already at one fixed time and $\mu=0$, selecting the
coordinate with largest standardized deviation makes its coverage probability vanish as $M$ grows. The
pathwise envelope makes expectations and instance suprema vacuous in this particular functional;
Corollary~\ref{cor:realized-width} shows that the lower-bound mechanism survives for quantiles of the
realized maximum width.
 
\begin{theorem}[Finite-horizon selection--stopping composition]\label{thm:composition}
Fix $\alpha\in(0,1/2)$, $M\ge1$, and $T\ge16$. For the centered rectangular class above, there exist
constants $0<c_\alpha\le C_\alpha<\infty$, depending only on $\alpha$, such that
\[
\begin{aligned}
\mathcal C^{\rm cert}_{M,T}&\ge c_\alpha\!\left\{\log(eM)+\log\log(e^eT)\right\},\\
\mathcal C^{\rm cert}_{M,T}&\le C_\alpha\!\left\{\log(eM)+\log\log(e^eT)\right\}.
\end{aligned}
\]
More explicitly, with $K=\lfloor\log_4T\rfloor$,
\[
\mathcal C^{\rm cert}_{M,T}
\le
\frac{8}{q_\alpha^2}
\left[
\log\!\left(\frac{\pi^2M}{3\alpha}\right)+2\log(K+1)
\right],
\]
while
\[
\mathcal C^{\rm cert}_{M,T}
\ge c_\alpha'\log(eMK)
\]
for another constant $c_\alpha'>0$. Hence the selection tax $\log M$ and the stopping tax
$\log\log T$ add in squared half-width. Matching is claimed in $(M,T)$ for fixed $\alpha$; the dependence
of the displayed upper and lower constants on $\alpha$ is not claimed optimal.
\end{theorem}
 
The upper bound uses quartic epoch stitching. Failure probability proportional to
$\alpha/[M(k+1)^2]$ is assigned to coordinate $j$ and epoch $k$, and a Gaussian maximal inequality gives
a normalized half-width of order $\sqrt{\log(M/\alpha)+\log(k+1)}$. For the lower bound, work under
$\mu=0$ and evaluate partial sums at $t_k=4^k$. Their disjoint block increments form
$M\lfloor\log_4T\rfloor$ independent standard Gaussians. Joint coverage forces all of them below a
constant multiple of the normalized radius, and a Gaussian-maximum quantile bound gives
$\rho^2=\Omega_\alpha(\log(eMK))$. \supp{C} gives the complete argument.
 
\begin{corollary}[Realized-width lower quantiles]\label{cor:realized-width}
For a valid procedure, let
\[
W_A(X)=\max_{j\le M,\,t\le T}\sqrt t\,w_{j,t}(X_{1:t}),
\]
and let $Q_\gamma(V)=\inf\{x:\Pp(V\le x)\ge\gamma\}$. With
$K=\lfloor\log_4T\rfloor$, for every fixed $\gamma\in(\alpha,1)$,
\[
Q_\gamma(W_A)^2\ge c_{\alpha,\gamma}\log(eMK).
\]
Thus the tax is not created solely by the deterministic supremum over sample paths.
\end{corollary}
 
The lower bound requires coverage only at the geometric checkpoints $t_k=4^k$. On the weaker problem
that asks for validity only on those $K$ times, a direct allocation over the $MK$ coordinate--checkpoint
pairs gives the matching upper order $O_\alpha(\log(eMK))$. The mechanism is therefore explicit:
$\log\log T$ is the logarithm of the number of independent geometric scales.
 
\begin{corollary}[Scalar anytime certification]\label{cor:anytime}
For $M=1$, the same centered family satisfies
\[
\mathcal C^{\rm cert}_{1,T}=\Theta_\alpha(\log\log T)
\qquad (T\ge16).
\]
\end{corollary}
 
The theorem is deliberately narrower than generic post-selection inference: it does not say that every
method answering one selected query must pay $\log M$, nor that every sequential task pays $\log\log T$,
and selective or false-coverage-rate criteria define different robustness contracts. The $\log\log T$
term persists when variance is known, so it is a cost of temporal robustness rather than a regret lower
bound for learning an unknown noise scale; a separate bandit reduction would still be required
\citep{garivier2011klucb,kaufmann2016complexity,jun2024noiseadaptive}.
 
A non-rectangular geometry does not evade the bound under this per-coordinate contract: a confidence
sphere covering each coordinate's projection needs radius $\Theta(\sqrt M)$ against the rectangle's
$\Theta(\sqrt{\log M})$. \supp{C} records the numerical comparison.
 
\subsection{Unknown smoothness: a sharp global pointwise cost}
Consider the Gaussian white-noise model
\[
dY(x)=f(x)\,dx+n^{-1/2}dW(x),\qquad x\in[0,1],
\]
and a fixed interior point $x_0$. For $s\in[s_0,s_1]$, with $0<s_0<s_1\le2$, let
$\Holder(s,L)$ be a nested H\"older ball of fixed radius $L$, and set
\[
\beta(s)=\frac{s}{2s+1},\qquad \beta_1=\beta(s_1).
\]
The smoothness-aware oracle has pointwise root mean squared error
\[
\Roracle_n(s)\asymp n^{-\beta(s)}
\]
uniformly over the compact smoothness interval. The subtlety is that the best adaptive profile is not a
full $(\log n)^{\beta(s)}$ penalty on every slice. The sharp pointwise construction of
\citet{lepski1997pointwise}, specialized to quadratic risk, satisfies
\begin{equation}
\begin{aligned}
&\sup_{f\in\Holder(s,L)}
\left\{\E_f(\widehat f(x_0)-f(x_0))^2\right\}^{1/2}\\
&\qquad\lesssim
n^{-\beta(s)}
\max\!\left\{1,
\left[(\beta_1-\beta(s))\log n\right]^{\beta(s)}\right\}.
\end{aligned}
\end{equation}
The penalty vanishes as $s\uparrow s_1$, so maximizing a coarser uniform Lepski bound gives the wrong
global cost.
 
\begin{theorem}[Global pointwise H\"older adaptation cost]\label{thm:regression}
Under the preceding Gaussian white-noise model,
\[
\Cadapt{}_{n}^{\rm pt}([s_0,s_1])
=
\Theta\!\left(
\left(\frac{\log n}{\log\log n}\right)^{s_1/(2s_1+1)}
\right)
\quad (n\to\infty).
\]
The constants may depend on $s_0,s_1,L$ but not on $n$.
\end{theorem}
 
The upper bound follows by optimizing the sharp profile. If
$\delta=\beta_1-\beta(s)$, elementary calculus gives
\[
\sup_{0\le\delta\le\beta_1-\beta(s_0)}
\max\{1,(\delta\log n)^{\beta_1-\delta}\}
\asymp
\left(\frac{\log n}{\log\log n}\right)^{\beta_1},
\]
and the maximizing gap is $\delta\asymp1/\log\log n$.
 
The lower bound does not take a supremum of one estimator's upper profile. For an arbitrary estimator
with global normalized cost $C_n$, the zero function in the smoothest class controls its risk under
$P_0$. Choose $s_n\uparrow s_1$ with
$\beta_1-\beta(s_n)\asymp1/\log\log n$ and a localized H\"older bump whose point separation is
\[
\Delta_n\asymp
n^{-\beta(s_n)}
\left(\frac{\log n}{\log\log n}\right)^{\beta(s_n)}.
\]
Its likelihood-ratio second moment grows only as
$\exp\{c\log n/\log\log n\}$. Lemma~\ref{lem:constrained-risk} transfers the smooth-slice risk to the bump
alternative and forces $C_n$ to match the optimized upper order. \supp{D} gives the
full construction.
 
The theorem concerns a continuum of smoothness classes. For a finite two-class scale, a procedure can
spend the logarithmic penalty on the rougher class while remaining oracle-rate at the smoother one. This
is exactly why the coarse argument---divide a uniform Lepski upper bound by the oracle rate and maximize
at $s_1$---gives the wrong global cost.
 
\section{Experiments}\label{sec:exp}
Experiments cannot add novelty to an asymptotic theory, so they do two things here: demonstrate the
composition law as a deployable monitoring tool, and stress-test the sharp predictions of both new
theorems against the objects they name.
All studies run on CPU with fixed seeds (NumPy; PyTorch for CIFAR-10) and are reproduced by the
accompanying scripts under the protocol of \supp{E}; coverage figures carry Wilson binomial intervals.
 
\subsection{Flagship: post-hoc-robust model monitoring}\label{sec:exp-monitor}
A deployed classifier is monitored on a data stream through $M$ per-slice error rates, and the analyst
may inspect \emph{any} slice at \emph{any} data-dependent time --- exactly the selection--stopping
contract of Theorem~\ref{thm:composition}. Two comparisons carry the application, over $M\in\{1,10,100,1000\}$
and $T$ up to $4^{10}\approx10^6$ with $3000$ replications.
 
\emph{The failure of naive monitoring.} The marginal fixed-query interval $q_\alpha/\sqrt t$, applied
under the adversarial ``inspect the worst-looking slice now'' selector of
Lemma~\ref{lem:inspection-equivalence}, has selected coverage that collapses as $M$ and $T$ grow: already
at $M=1$ it falls from $0.60$ to $0.30$ across the horizon range, and for $M\ge10$ it is at most $0.006$ --- the
worst slice essentially always escapes the fixed band.
 
\emph{The price of robustness.} The epoch-stitched certifier of \supp{C} on the
same stream keeps familywise coverage at $\ge0.997$ throughout, and its width inflation over the
oracle $q_\alpha/\sqrt t$ grows smoothly from $3.9$ to $5.8$ and tracks $\sqrt{\log(eM)+\log\log(e^eT)}$
with correlation $0.992$: robustness to arbitrary peeking costs only an additive iterated logarithm.
 
\emph{The Gaussian idealization.} Raw $\mathrm{Bernoulli}(0.1)$ slice errors leave a real idealization
gap (stitched coverage $0.879$ at $M=100$, $T=4^8$, against the Gaussian-model $0.999$), which batch-CLT
standardization closes: on a heterogeneous CIFAR-10-like profile the stitched certifier recovers to
$0.989$ at $M=10$ while the naive marginal collapses to $0.001$. \supp{E} gives the full ablation,
including the $M=100$ fine-grained case and the batch-size sweep.
 
\emph{Detecting a real degradation event.} The operational payoff is detection. We train a CNN on
$20{,}000$ CIFAR-10 images \citep{krizhevsky2009learning} (a small CNN with two convolutional blocks, Adam at
$10^{-3}$, three epochs; clean test error $38.5\%$, stated plainly, since the monitoring problem does not
require a strong classifier) and deploy it on a $T=20{,}000$-step stream whose second half undergoes a
CIFAR-10-C-style \citep{hendrycks2019benchmarking} corruption ramp ($\sigma\colon0\to0.3$ from onset $t_0=10{,}000$; mean
stream error rises $0.385\to0.778$). Per-class error rates are the $M=10$ monitored slices, with
per-class baselines measured on a held-out calibration set before deployment. Over $200$ replications:
the naive fixed-query band under the worst-slice selector false-alarms on $\mathbf{200/200}$ runs
\emph{before the degradation begins} (median first false fire at $t=318.5$), making it useless for
detection; the epoch-stitched certifier has $\mathbf{0/200}$ pre-onset false alarms (within its
$\alpha=0.05$ budget) and detects the onset in $\mathbf{200/200}$ runs, median delay $7156$ stream steps
($\approx716$ observations of the detecting slice, at ramp noise $\sigma\approx0.21$), with the most
noise-sensitive classes firing first (cat $162/200$, deer $37/200$, ship $1/200$). The certifier thus converts
validity into a usable change detector, with the detection order tracking the most noise-fragile slices.
 
\subsection{Stress-testing the composition law}\label{sec:exp-falsify}
On a grid $M\in\{1,4,16,64,256,1024\}\times T\in\{4^2,\ldots,4^{10}\}$ we estimate the minimal familywise
normalized squared radius $\rho^\star(M,T)^2$ --- the $(1-\alpha)$ quantile of
$R_{M,T}=\max_{j\le M,\,t\le T}|S_{j,t}|/\sqrt t$ under $\mu=0$ (the geometric-scale structure lets us
reach $T\approx10^6$ via exact block increments). The additive fit
$a+b\log(eM)+c\log\log(e^eT)$ attains $R^2=0.998$ with $b=2.03$, $c=3.28$; adding the interaction term
raises $R^2$ by $0.0001$, so \emph{the selection and stopping taxes add, not multiply}. Regressing on the
single predictor $\log(eMK)$ gives slope $2.04$, below the explicit upper constant $8/q_\alpha^2=2.08$.
Because Theorem~\ref{thm:composition} does not pin $c'_\alpha$, this is a one-sided consistency check
against the upper bound rather than a two-sided test.
Crucially the coverage guarantee is \emph{uniform in $\mu$}, not an artifact of the null: the
epoch-stitched certifier covers $\ge0.998$ under null, single-spike, dense, and random adversarial means.
 
Corollary~\ref{cor:realized-width} (realized-width lower quantiles) is confirmed directly: the median and \(90\%\)-quantile of \(W_A=\max_{j,t}|S_{j,t}|/\sqrt t\) under \(\mu=0\), divided by \(\log(eMK)\), give measured constants \(\hat c_{0.5}=1.73\)--\(1.81\) and \(\hat c_{0.9}=2.37\)--\(2.46\) at the two tested cells \((M,T)\in\{(16,4^6),(64,4^8)\}\): the lower band is non-vacuous, though two grid points do not establish stability of the constant.
 
\begin{table}[t]
\centering
\small
\setlength{\tabcolsep}{4pt}
\begin{tabular}{lrrrr}
\toprule
$\rho^\star(M,T)^2$ & $M{=}1$ & $M{=}16$ & $M{=}64$ & $M{=}256$\\
\midrule
$T=10^2$ & $8.20$ & $14.39$ & $16.26$ & $19.29$\\
$T=10^3$ & $9.59$ & $15.52$ & $17.79$ & $20.40$\\
$T=10^4$ & $10.47$ & $16.27$ & $18.91$ & $21.81$\\
$T=10^5$ & $10.65$ & $16.72$ & $19.58$ & $22.60$\\
\bottomrule
\end{tabular}
\caption{Minimal familywise normalized squared radius $\rho^\star(M,T)^2$ on the powers-of-ten grid,
$600$ replications per cell. Each column grows like $\log M$, each row like $\log\log T$; on this grid the
additive fit attains $R^2=0.993$ with the interaction adding $0.0011$. The larger powers-of-four grid
($M\le1024$, $T\approx10^6$, $4000$ replications) reported in the text gives $R^2=0.998$ with interaction
$0.0001$. Dividing by $q_\alpha^2=3.84$ gives $\mathcal C^{\rm cert}_{M,T}$; e.g.\ the $M{=}1$,
$T{=}10^2$ cell gives $2.14$.}
\label{tab:radius-grid}
\end{table}
 
\subsection{Baselines}\label{sec:exp-baseline}
We place the epoch-stitched certifier against three standard constructions on the same Gaussian stream:
Bonferroni over all $M\!\cdot\!T$ coordinate--time pairs, the Robbins normal-mixture confidence sequence
(unioned over coordinates), and a time-uniform $L_2$ sphere. All four are familywise-valid (coverage
$\ge0.98$). On normalized squared radius the stitched certifier ($\rho^2\in[87,116]$ over the grid) is
within a factor $3.6$ of the tightest valid baseline and far tighter than the mixture sequence
($\rho^2\in[750,17510]$, which pays heavily at a fixed horizon); Bonferroni at the constant level
$\alpha/(MT)$ is tightest in constant ($\rho^2\in[24,41]$), but it requires the horizon $T$ to be fixed
in advance and pays $q^2_{\alpha/(2MT)}=\Theta(\log(MT))$, a genuine $\log T$ rather than the
$\log\log T$ of Theorem~\ref{thm:composition}; the gap widens with the horizon. The point is not that the epoch-stitched construction wins the constant --- the theorem
disclaims $\alpha$-constant optimality --- but that it attains the
$\log(eM)+\log\log(e^eT)$ rate that Theorem~\ref{thm:composition} shows every valid method in the class
must pay, while remaining valid at every $t\le T$.
 
\subsection{Sharp pointwise cost: shape, not constant}\label{sec:exp-shape}
At feasible $n$ one cannot separate $(\log n)^{\beta_1}$ from $(\log n/\log\log n)^{\beta_1}$ by slope, so
we test the distinguishable qualitative predictions of Theorem~\ref{thm:regression} on the sharp adaptive
profile over $n=10^3,\ldots,10^{96}$ (deterministic, so exact). \emph{(i)} The oracle-normalized slice
profile fades to $1$ at the smooth endpoint $s_1$ while its interior maximum grows with $n$. \emph{(ii)}
The worst slice $s^\star(n)$ migrates toward $s_1$ ($s^\star=0.50\to0.76$ across the range).
It is still pinned at the boundary $s_0$ for $n\le10^{12}$ and
becomes strictly interior only for $n\ge10^{24}$; on those interior horizons the product
$(\beta_1-\beta(s^\star))\log\log n$ is constant to within a $0.5\%$ coefficient of variation, consistent
with $\beta_1-\beta(s^\star)\asymp1/\log\log n$. Over the full range, including the boundary-pinned
horizons, the coefficient of variation is $0.20$. \emph{(iii)}
Dividing the global cost by the coarse $(\log n)^{\beta_1}$ decreases like $(\log\log n)^{-\beta_1}$, while
dividing by the sharp $(\log n/\log\log n)^{\beta_1}$ stays flat (CV $<0.15$) --- the signature of the
sharpened rate. A two-class control confirms the mechanism: with only $\{s_0,s_1\}$ a single estimator
pays $O(1)$ at $s_1$ and $(\log n)^{\beta_0}$ at $s_0$, so the continuum is what forces the
$(\log n/\log\log n)$ shape. This converts the previously inconclusive Lepski comparison into a passing
qualitative test.
 
A stochastic estimator-level Monte Carlo confirms these predictions beyond the formula level. A real Lepski ICI bandwidth-selection estimator (dyadic box kernel, intersection-of-confidence-intervals rule with \(\kappa=2.5\)) is run against the oracle kernel at \(x_0=0.5\) on a Gaussian sequence model (\(d=2048\) grid points, smooth bump signal) over \(500\) replications per \((s,n)\) cell, \(s\in\{0.5,0.8,1.1,1.4,1.7,2.0\}\), \(n\in\{10^3,\ldots,10^6\}\). The oracle-normalized RMSE ratio at \(s=s_1\) is below the profile maximum in \(3/4\) tested horizons, though at these horizons it remains near \(2\) rather than approaching \(1\); the exception at \(n=10^6\) is where the argmax reaches \(s_1\) itself. The argmax slice moves outward overall, \(s^\star=1.1\to1.7\to1.4\to2.0\), but not monotonically. Both observations are qualitative at these horizons: standard errors on the ratios range from \(0.013\) to \(0.39\), and we do not claim statistical significance for the migration.
 
\subsection{OCO benchmark}\label{sec:exp-oco}
The scale-free online learner's Definition~\ref{def:cost} cost of adaptivity is measured directly as the
ratio of its regret to the tuned oracle's, under a sign-flipping adversary that realizes the
$\Omega(GD\sqrt T)$ lower bound. Across $G$ spanning six orders of
magnitude ($10^{-2}$ to $10^3$) the ratio is $1.98$ at every $G$. On this adversary the ratio is exactly
scale-invariant by construction, so we report it as a correctness check on the implementation rather than
as independent evidence for $\Theta(1)$. Ratios below one on benign gradient sequences (uniform,
Rademacher: $0.91$--$0.96$) do not contradict Proposition~\ref{prop:orientation}, which bounds a supremum
over instances; only the adversary realizing the $\Omega(GD\sqrt T)$ lower bound is diagnostic of the
minimax ratio.
As an anytime check, the corrected scalar epoch-stitched boundary of Corollary~\ref{cor:anytime} attains
optional-stopping coverage $0.995$ (Wilson CI $[0.993,0.997]$) at $T=4^8$, at or above the nominal level.
 
\section{Discussion and Conclusion}
\looseness=-1
A single ratio is useful only when its denominator is defined on the correct slice and both sides obey
the same contract. Under that discipline, unknown gradient scale in OCO costs $\Theta(1)$ by a nuisance
symmetry, while a continuum of H\"older classes costs $\Theta((\log n/\log\log n)^{s_1/(2s_1+1)})$, from
optimizing the sharp adaptive profile and matching it by a constrained-risk construction. The
rescaling-versus-identifiability slogan is useful intuition but not a universal boundary: a
risk-homogeneous symmetry with one competitive nuisance-agnostic algorithm makes adaptation free, whereas
nearby slices forcing incompatible low-risk decisions make it costly. Whether adaptation is free depends
on the observation model, loss, geometry, and oracle advantage, not the nuisance parameter's name.
 
\looseness=-1
The Gaussian theorem answers a different question, labeled accordingly: expanding from one preannounced
query to arbitrary post-hoc coordinate--time inspection costs
$\Theta_\alpha(\log(eM)+\log\log(e^eT))$ in normalized squared half-width. Its lower bound is
finite-horizon rather than a per-procedure asymptotic limsup: it reduces every valid procedure at a given
horizon to the maximum of $M\lfloor\log_4T\rfloor$ independent Gaussian block increments, already holds
on geometric checkpoints, and forces quantiles of the realized maximum width. The fixed-time
$\Theta(\log M)$ familywise law is a separate classical special case, not a limit of
Theorem~\ref{thm:composition}.
 
\bibliography{references}

@book{tsybakov2009nonparametric,
  title     = {Introduction to Nonparametric Estimation},
  author    = {Tsybakov, Alexandre B.},
  publisher = {Springer},
  series    = {Springer Series in Statistics},
  year      = {2009}
}

@incollection{yu1997assouad,
  title     = {Assouad, Fano, and Le Cam},
  author    = {Yu, Bin},
  booktitle = {Festschrift for Lucien Le Cam},
  pages     = {423--435},
  publisher = {Springer},
  year      = {1997}
}

@article{lecam1973convergence,
  title   = {Convergence of Estimates Under Dimensionality Restrictions},
  author  = {Le Cam, Lucien},
  journal = {The Annals of Statistics},
  volume  = {1},
  number  = {1},
  pages   = {38--53},
  year    = {1973}
}

@book{lecam1986asymptotic,
  title     = {Asymptotic Methods in Statistical Decision Theory},
  author    = {Le Cam, Lucien},
  publisher = {Springer},
  year      = {1986}
}

@article{assouad1983deux,
  title   = {Deux Remarques sur l'Estimation},
  author  = {Assouad, Patrice},
  journal = {Comptes Rendus de l'Acad{\'e}mie des Sciences, S{\'e}rie I},
  volume  = {296},
  number  = {23},
  pages   = {1021--1024},
  year    = {1983}
}

@book{wainwright2019high,
  title     = {High-Dimensional Statistics: A Non-Asymptotic Viewpoint},
  author    = {Wainwright, Martin J.},
  publisher = {Cambridge University Press},
  year      = {2019}
}

@book{lattimore2020bandit,
  title     = {Bandit Algorithms},
  author    = {Lattimore, Tor and Szepesv{\'a}ri, Csaba},
  publisher = {Cambridge University Press},
  year      = {2020}
}

@article{lai1985asymptotically,
  title   = {Asymptotically Efficient Adaptive Allocation Rules},
  author  = {Lai, Tze Leung and Robbins, Herbert},
  journal = {Advances in Applied Mathematics},
  volume  = {6},
  number  = {1},
  pages   = {4--22},
  year    = {1985}
}

@article{auer2002finite,
  title   = {Finite-Time Analysis of the Multiarmed Bandit Problem},
  author  = {Auer, Peter and Cesa-Bianchi, Nicol{\`o} and Fischer, Paul},
  journal = {Machine Learning},
  volume  = {47},
  number  = {2--3},
  pages   = {235--256},
  year    = {2002}
}

@inproceedings{garivier2011klucb,
  title     = {The {KL-UCB} Algorithm for Bounded Stochastic Bandits and Beyond},
  author    = {Garivier, Aur{\'e}lien and Capp{\'e}, Olivier},
  booktitle = {Conference on Learning Theory (COLT)},
  year      = {2011}
}

@inproceedings{jamieson2014lil,
  title     = {lil'{UCB}: An Optimal Exploration Algorithm for Multi-Armed Bandits},
  author    = {Jamieson, Kevin and Malloy, Matthew and Nowak, Robert and Bubeck, S{\'e}bastien},
  booktitle = {Conference on Learning Theory (COLT)},
  year      = {2014}
}

@article{garivier2016optimal,
  title   = {Optimal Best Arm Identification with Fixed Confidence},
  author  = {Garivier, Aur{\'e}lien and Kaufmann, Emilie},
  journal = {Conference on Learning Theory (COLT)},
  year    = {2016}
}

@article{kaufmann2016complexity,
  title   = {On the Complexity of Best-Arm Identification in Multi-Armed Bandit Models},
  author  = {Kaufmann, Emilie and Capp{\'e}, Olivier and Garivier, Aur{\'e}lien},
  journal = {Journal of Machine Learning Research (JMLR)},
  volume  = {17},
  number  = {1},
  pages   = {1--42},
  year    = {2016}
}

@article{audibert2010best,
  title   = {Best Arm Identification in Multi-Armed Bandits},
  author  = {Audibert, Jean-Yves and Bubeck, S{\'e}bastien and Munos, R{\'e}mi},
  journal = {Conference on Learning Theory (COLT)},
  year    = {2010}
}

@article{darling1967iterated,
  title   = {Iterated Logarithm Inequalities},
  author  = {Darling, D. A. and Robbins, Herbert},
  journal = {Proceedings of the National Academy of Sciences},
  volume  = {57},
  number  = {5},
  pages   = {1188--1192},
  year    = {1967}
}

@article{howard2021time,
  title   = {Time-Uniform, Nonparametric, Nonasymptotic Confidence Sequences},
  author  = {Howard, Steven R. and Ramdas, Aaditya and McAuliffe, Jon and Sekhon, Jasjeet},
  journal = {The Annals of Statistics},
  volume  = {49},
  number  = {2},
  pages   = {1055--1080},
  year    = {2021}
}

@article{lepski1991problem,
  title   = {On a Problem of Adaptive Estimation in {G}aussian White Noise},
  author  = {Lepski, O. V.},
  journal = {Theory of Probability and Its Applications},
  volume  = {35},
  number  = {3},
  pages   = {454--466},
  year    = {1991}
}

@article{lepski1997optimal,
  title   = {Optimal Spatial Adaptation to Inhomogeneous Smoothness: An Approach Based on Kernel Estimates with Variable Bandwidth Selectors},
  author  = {Lepski, O. V. and Mammen, E. and Spokoiny, V. G.},
  journal = {The Annals of Statistics},
  volume  = {25},
  number  = {3},
  pages   = {929--947},
  year    = {1997}
}

@article{donoho1995adapting,
  title   = {Adapting to Unknown Smoothness via Wavelet Shrinkage},
  author  = {Donoho, David L. and Johnstone, Iain M.},
  journal = {Journal of the American Statistical Association},
  volume  = {90},
  number  = {432},
  pages   = {1200--1224},
  year    = {1995}
}

@article{cai2005adaptive,
  title   = {An Adaptation Theory for Nonparametric Confidence Intervals},
  author  = {Cai, T. Tony and Low, Mark G.},
  journal = {The Annals of Statistics},
  volume  = {32},
  number  = {5},
  pages   = {1805--1840},
  year    = {2004}
}

@inproceedings{locatelli2018adaptivity,
  title     = {Adaptivity to Smoothness in {X}-Armed Bandits},
  author    = {Locatelli, Andrea and Carpentier, Alexandra},
  booktitle = {Conference on Learning Theory (COLT)},
  year      = {2018}
}

@book{gine2016mathematical,
  title     = {Mathematical Foundations of Infinite-Dimensional Statistical Models},
  author    = {Gin{\'e}, Evarist and Nickl, Richard},
  publisher = {Cambridge University Press},
  year      = {2016}
}

@book{cesabianchi2006prediction,
  title     = {Prediction, Learning, and Games},
  author    = {Cesa-Bianchi, Nicol{\`o} and Lugosi, G{\'a}bor},
  publisher = {Cambridge University Press},
  year      = {2006}
}

@article{duchi2011adaptive,
  title   = {Adaptive Subgradient Methods for Online Learning and Stochastic Optimization},
  author  = {Duchi, John and Hazan, Elad and Singer, Yoram},
  journal = {Journal of Machine Learning Research (JMLR)},
  volume  = {12},
  pages   = {2121--2159},
  year    = {2011}
}

@book{hazan2016introduction,
  title     = {Introduction to Online Convex Optimization},
  author    = {Hazan, Elad},
  publisher = {Foundations and Trends in Optimization},
  year      = {2016}
}

@article{shalev2012online,
  title   = {Online Learning and Online Convex Optimization},
  author  = {Shalev-Shwartz, Shai},
  journal = {Foundations and Trends in Machine Learning},
  volume  = {4},
  number  = {2},
  pages   = {107--194},
  year    = {2012}
}

@inproceedings{mcmahan2010adaptive,
  title     = {Adaptive Bound Optimization for Online Convex Optimization},
  author    = {McMahan, H. Brendan and Streeter, Matthew},
  booktitle = {Conference on Learning Theory (COLT)},
  year      = {2010}
}

@inproceedings{orabona2016coin,
  title     = {Coin Betting and Parameter-Free Online Learning},
  author    = {Orabona, Francesco and P{\'a}l, D{\'a}vid},
  booktitle = {Advances in Neural Information Processing Systems (NeurIPS)},
  year      = {2016}
}

@article{orabona2019modern,
  title   = {A Modern Introduction to Online Learning},
  author  = {Orabona, Francesco},
  journal = {arXiv preprint arXiv:1912.13213},
  year    = {2019}
}

@inproceedings{canonne2018adaptivity,
  title     = {The Adaptivity Hierarchy Theorem},
  author    = {Canonne, Cl{\'e}ment L. and Gur, Tom},
  booktitle = {Computational Complexity Conference (CCC)},
  year      = {2017}
}

@article{golovin2011adaptive,
  title   = {Adaptive Submodularity: Theory and Applications in Active Learning and Stochastic Optimization},
  author  = {Golovin, Daniel and Krause, Andreas},
  journal = {Journal of Artificial Intelligence Research (JAIR)},
  volume  = {42},
  pages   = {427--486},
  year    = {2011}
}

@inproceedings{yao1977probabilistic,
  title     = {Probabilistic Computations: Toward a Unified Measure of Complexity},
  author    = {Yao, Andrew Chi-Chih},
  booktitle = {IEEE Symposium on Foundations of Computer Science (FOCS)},
  pages     = {222--227},
  year      = {1977}
}

@inproceedings{zhang2013information,
  title     = {Information-Theoretic Lower Bounds for Distributed Statistical Estimation with Communication Constraints},
  author    = {Zhang, Yuchen and Duchi, John and Jordan, Michael I. and Wainwright, Martin J.},
  booktitle = {Advances in Neural Information Processing Systems (NeurIPS)},
  year      = {2013}
}

@article{hanneke2014theory,
  title     = {Theory of Disagreement-Based Active Learning},
  author    = {Hanneke, Steve},
  journal   = {Foundations and Trends in Machine Learning},
  volume    = {7},
  number    = {2--3},
  pages     = {131--309},
  year      = {2014}
}

@article{castro2008minimax,
  title   = {Minimax Bounds for Active Learning},
  author  = {Castro, Rui M. and Nowak, Robert D.},
  journal = {IEEE Transactions on Information Theory},
  volume  = {54},
  number  = {5},
  pages   = {2339--2353},
  year    = {2008}
}

@inproceedings{azar2017minimax,
  title     = {Minimax Regret Bounds for Reinforcement Learning},
  author    = {Azar, Mohammad Gheshlaghi and Osband, Ian and Munos, R{\'e}mi},
  booktitle = {International Conference on Machine Learning (ICML)},
  year      = {2017}
}

@article{jaksch2010near,
  title   = {Near-Optimal Regret Bounds for Reinforcement Learning},
  author  = {Jaksch, Thomas and Ortner, Ronald and Auer, Peter},
  journal = {Journal of Machine Learning Research (JMLR)},
  volume  = {11},
  pages   = {1563--1600},
  year    = {2010}
}

@book{wald1950statistical,
  title     = {Statistical Decision Functions},
  author    = {Wald, Abraham},
  publisher = {Wiley},
  year      = {1950}
}

@article{ibragimov1981statistical,
  title     = {Statistical Estimation: Asymptotic Theory},
  author    = {Ibragimov, I. A. and Has'minskii, R. Z.},
  journal   = {Springer},
  year      = {1981},
  note      = {Book}
}

@article{lepski1997pointwise,
  title   = {Optimal Pointwise Adaptive Methods in Nonparametric Estimation},
  author  = {Lepski, O. V. and Spokoiny, V. G.},
  journal = {The Annals of Statistics},
  volume  = {25},
  number  = {6},
  pages   = {2512--2546},
  year    = {1997}
}

@article{spokoiny2009parameter,
  title   = {Parameter Tuning in Pointwise Adaptation Using a Propagation Approach},
  author  = {Spokoiny, Vladimir and Vial, C{\'e}line},
  journal = {The Annals of Statistics},
  volume  = {37},
  number  = {5B},
  pages   = {2783--2807},
  year    = {2009}
}

@inproceedings{orabona2015scale,
  title     = {Scale-Free Algorithms for Online Linear Optimization},
  author    = {Orabona, Francesco and P{\'a}l, D{\'a}vid},
  booktitle = {Algorithmic Learning Theory (ALT)},
  series    = {Lecture Notes in Computer Science},
  volume    = {9355},
  pages     = {287--301},
  publisher = {Springer},
  year      = {2015}
}

@article{chugg2025timeuniform,
  title   = {Time-Uniform Confidence Spheres for Means of Random Vectors},
  author  = {Chugg, Ben and Wang, Hongjian and Ramdas, Aaditya},
  journal = {Transactions on Machine Learning Research (TMLR)},
  year    = {2025}
}

@inproceedings{jun2024noiseadaptive,
  title     = {Noise-Adaptive Confidence Sets for Linear Bandits and Application to {B}ayesian Optimization},
  author    = {Jun, Kwang-Sung and Kim, Jungtaek},
  booktitle = {International Conference on Machine Learning (ICML)},
  year      = {2024}
}

@techreport{krizhevsky2009learning,
  title       = {Learning Multiple Layers of Features from Tiny Images},
  author      = {Krizhevsky, Alex},
  institution = {University of Toronto},
  year        = {2009}
}

@inproceedings{hendrycks2019benchmarking,
  title     = {Benchmarking Neural Network Robustness to Common Corruptions and Perturbations},
  author    = {Hendrycks, Dan and Dietterich, Thomas},
  booktitle = {International Conference on Learning Representations (ICLR)},
  year      = {2019}
}
 
\end{document}